\documentclass{article}

\usepackage{arxiv}

\usepackage[draft]{graphicx}

\usepackage{microtype}
\usepackage{subfigure}
\usepackage{booktabs}
\usepackage{comment}
\usepackage{wrapfig}
\usepackage{arydshln}
\usepackage{enumitem}
\usepackage{multirow}
\usepackage{url}
\usepackage{natbib}
\usepackage{authblk}

\RequirePackage[colorlinks,citecolor=blue,linkcolor=blue,urlcolor=blue,pagebackref]{hyperref}

\usepackage{amsmath}
\usepackage{amssymb}
\usepackage{mathtools}
\usepackage{amsfonts}
\usepackage{bm}
\usepackage{bbm}

\usepackage{algorithm}
\usepackage{algorithmic}

\def\eqref#1{equation~\ref{#1}}

\def\1{\bm{1}}

\def\bmmu{{\bm{\mu}}}
\def\bmnu{{\bm{\nu}}}

\DeclareMathAlphabet{\mathsfit}{\encodingdefault}{\sfdefault}{m}{sl}
\SetMathAlphabet{\mathsfit}{bold}{\encodingdefault}{\sfdefault}{bx}{n}

\def\calF{{\mathcal{F}}}

\def\calP{{\mathcal{P}}}

\def\bbE{{\mathbb{E}}}

\def\bbP{{\mathbb{P}}}

\def\bbR{{\mathbb{R}}}

\newcommand{\Var}{\mathrm{Var}}

\DeclareMathOperator*{\argmax}{arg\,max}

\newcommand{\p}[1]{\left(#1\right)}
\newcommand{\sqb}[1]{\left[#1\right]}
\newcommand{\cb}[1]{\left\{#1\right\}}

\newcommand{\Bigp}[1]{\Big(#1\Big)}

\newcommand{\Bigcb}[1]{\Big\{#1\Big\}}

\newcommand{\Biggp}[1]{\Bigg(#1\Bigg)}

\newcommand{\bigp}[1]{\big(#1\big)}
\newcommand{\bigsqb}[1]{\big[#1\big]}

\def\Regret{{\mathrm{Regret}}}
\usepackage{amsthm}

\theoremstyle{plain}

\newtheorem{theorem}{Theorem}[section]
\newtheorem{lemma}[theorem]{Lemma}

\newtheorem{corollary}[theorem]{Corollary}
\newtheorem{proposition}[theorem]{Proposition}

\newtheorem{definition}[theorem]{Definition}

\usepackage[textsize=tiny]{todonotes}

\renewcommand{\eqref}[1]{(\ref{#1})}

\usepackage{xcolor}
\newcount\Comments  
\newcommand{\kibitz}[2]{\ifnum\Comments=1\textcolor{#1}{#2}\fi}

\newcommand*{\email}[1]{\texttt{#1}}

\usepackage[capitalize,noabbrev]{cleveref}

\allowdisplaybreaks

\title{Minimax Optimal Simple Regret\\ in Two-Armed Best-Arm Identification}

\renewcommand{\shorttitle}{}

\author{Masahiro Kato}

\affil{Department of Basic Science, The University of Tokyo\\
\email{mkato-csecon@g.ecc.u-tokyo.ac.jp}}

\begin{document}

\maketitle

\begin{abstract}
This study investigates an asymptotically minimax optimal algorithm in the two-armed fixed-budget best-arm identification (BAI) problem. Given two treatment arms, the objective is to identify the arm with the highest expected outcome through an adaptive experiment. We focus on the Neyman allocation, where treatment arms are allocated following the ratio of their outcome standard deviations. Our primary contribution is to prove the minimax optimality of the Neyman allocation for the simple regret, defined as the difference between the expected outcomes of the true best arm and the estimated best arm. Specifically, we first derive a minimax lower bound for the expected simple regret, which characterizes the worst-case performance achievable under the location-shift distributions, including Gaussian distributions. We then rigorously show that the simple regret of the Neyman allocation asymptotically matches this lower bound, including the constant term, not just the rate in terms of the sample size, under the worst-case distribution. Notably, our optimality result holds without imposing locality restrictions on the distribution, such as the local asymptotic normality. Furthermore, we demonstrate that the Neyman allocation reduces to the uniform allocation, i.e., the standard randomized controlled trial, under Bernoulli distributions.
\end{abstract}

\section{Introduction}
We address the problem of adaptive experimental design with two treatment arms, where the goal is to identify the treatment arm with the highest expected outcome through an adaptive experiment. This problem, often referred to as the fixed-budget best-arm identification (BAI) problem, has been widely studied in various fields, including machine learning \citep{Audibert2010,Bubeck2011}, operations research, economics \citep{Kasy2021}, and epidemiology. 

In this study, we focus on the Neyman allocation algorithm, which allocates samples to the treatment arms following the ratio of their standard deviations. We prove that the Neyman allocation is asymptotically minimax optimal for the simple regret, which is the difference between the expected outcomes of the true best arm and the estimated best arm. 

While it is known that the Neyman allocation achieves asymptotic optimality for \emph{any} distribution \citep{glynn2004large, Kaufman2016complexity}, including worst-case scenarios, the optimal algorithm has remained unknown when the outcome variances are unknown.

Our contributions are twofold. First, we derive a minimax lower bound for the simple regret under the worst-case distribution among all distributions with fixed variances. Second, we demonstrate that the Neyman allocation achieves this minimax lower bound asymptotically, including the constant term, thereby providing a complete solution to the problem. Notably, our results hold without requiring any locality restrictions on the distributions.

The remainder of this paper is organized as follows. In this section, we provide a formal problem setup, contributions, and related work. Section~\ref{sec:neyman} defines the Neyman allocation. Section~\ref{sec:lower} presents the minimax lower bound, including its derivation of the minimax lower bound. Section~\ref{sec:upper} shows the regret upper bound of the Neyman allocation and proves the minimax optimality by demonstrating that the upper bound matches the lower bound. 

\subsection{Problem setting}
\label{sec:prob}
We formulate the problem as follows. There are two arms, and each arm $a \in \{1, 2\}$ has a potential outcome $Y(a) \in \mathbb{R}$. Each potential outcome follows a marginal distribution $P_{\mu(a)}(a)$, and let $P_{\bmmu} \coloneqq (P_{\mu(a)}(1), P_{\mu(a)}(2))$ be the pair of the marginal distributions, where $\bmmu \coloneqq \{\mu(1), \mu(2)\} \in \mathbb{R}^2$ represents the set of mean parameters of $(Y(1), Y(2))$. Specifically, the expected value of each outcome satisfies $\bbE_{\bmmu}[Y(a)] = \mu(a)$, where $\bbE_{\bmmu}[\cdot]$ is the expectation under $P_{\bmmu}$.

Let $\bmmu_0 \coloneqq \{\mu_0(1), \mu_0(2)\}$ represent the \emph{true} mean parameters. The objective is to identify the best arm
\[
    a^*(\bmmu_0) = \arg\max_{a \in \{1, 2\}} \mu_0(a)
\]
through an adaptive experiment where data is generated from $P_{\bmmu_0}$.

Let $T$ denote the total sample size, also referred to as the budget. We consider an adaptive experimental procedure consisting of two phases:

\begin{description}
    \item[(1) Allocation phase:] For each $t \in [T] \coloneqq \{1, 2, \dots, T\}$:
    \begin{itemize}
        \item A treatment arm $A_t \in \{1, 2\}$ is selected based on the past observations $\{(A_s, Y_s)\}_{s=1}^{t-1}$.
        \item The corresponding outcome $Y_t$ is observed, where \[
        Y_t \coloneqq \sum_{a \in \{1, 2\}} Y_t(a),
        \] and $(Y_t(1), Y_t(2))$ follows the distribution $P_{\bmmu_0}$.
    \end{itemize}
    \item[(2) Recommendation phase:] At the end of the experiment ($t = T$), based on the observed outcomes, the best arm $\widehat{a}_T \in \{1, 2\}$ is recommended as the estimate of $a_0^*$.
\end{description}

Our task is to design an algorithm $\pi$ that determines how arms are selected during the allocation phase and how the best arm is recommended at the end of the experiment. An algorithm $\pi$ is formally defined as a pair $((A_t^{\pi})_{t \in [T]}, \widehat{a}_T^{\pi})$, where $(A_t^{\pi})_{t \in [T]}$ are indicators for the selected arms, and $\widehat{a}_T^{\pi}$ is the estimator of the best arm $a_0^*$. For simplicity, we omit the subscript $\pi$ when the dependence is clear from the context.

The performance of an algorithm $\pi$ is measured by the expected simple regret, defined as:
\begin{align*}
    \mathrm{Regret}_{P_{\bmmu_0}}(\pi) \coloneqq \bbE\left[ Y\bigp{a^*(\bmmu_0)} - Y\bigp{\widehat{a}_T^{\pi}} \right] = \mu_0\bigp{a^*(\bmmu_0)} - \mu_0\bigp{\widehat{a}_T^{\pi}}.
\end{align*}
In other words, the goal is to design an algorithm $\pi$ that minimizes the simple regret $\mathrm{Regret}_{P_{\bmmu_0}}(\pi)$.

\paragraph{Notation.} Let $\bbP_{P_{\bmmu}}$ denote the probability law under $P_{\bmmu}$, and let $\bbE_{P_{\bmmu}}$ represent the corresponding expectation operator. For notational simplicity, depending on the context, we abbreviate $\bbP_{P_{\bmmu}}[\cdot]$, $\bbE_{P_{\bmmu}}[\cdot]$, and $\mathrm{Regret}_{P_{\bmmu}}(\pi)$ as $\bbP_{\bmmu}[\cdot]$, $\bbE_{\bmmu}[\cdot]$, and $\mathrm{Regret}_{\bmmu}(\pi)$, respectively.

For each $a \in \{1, 2\}$, let $P_{a, \bmmu}$ denote the marginal distribution of $Y(a)$ under $P_{\bmmu}$. The Kullback-Leibler (KL) divergence between two distributions $P_{a, \bmmu}$ and $P_{a, \bmnu}$, where $\bmmu, \bmnu \in \mathbb{R}^2$, is denoted as $\mathrm{KL}(P_{a, \bmmu}, P_{a, \bmnu})$. When the marginal distribution depends only on the parameters $\mu(a)$ and $\nu(a)$, we simplify the notation to $\mathrm{KL}(\mu(a), \nu(a))$. Let $\mathcal{F}_{t} = \sigma(A_1, Y_1, \ldots, A_t, Y_t)$ be the sigma-algebras. 

For simplicity, we refer to the expected simple regret as the simple regret in this study, although the simple regret originally refers to the random variable $Y\bigp{a^*(\bmmu_0)} - Y\bigp{\widehat{a}_T^{\pi}}$ without expectation.

\subsection{Content of the Paper}
This study proposes an asymptotically minimax optimal algorithm by deriving a minimax lower bound and demonstrating that the simple regret of the proposed algorithm exactly matches the lower bound, including the constant term, not only the rate with respect to $T$.

First, we define the Neyman allocation in Section~\ref{sec:neyman}. Since the variance is unknown, we estimate it adaptively during the experiment. In the recommendation phase, we employ the augmented inverse probability weighting (AIPW) estimator. The AIPW estimator is chosen because it simplifies the theoretical analysis due to its unbiasedness property for the average treatment effect (ATE), while also being known for achieving the smallest variance.

Next, we develop a minimax lower bound. Let $\mathcal{P}_{\bm{\sigma}^2}$ be the class of distributions with fixed variances, formally defined in Definition~\ref{def:location}. We prove that the simple regret of any algorithm that asymptotically identifies the best arm with probability one (Definition~\ref{def:consistent}) cannot improve upon the following lower bound:
\begin{align*}
    &\inf_{\pi \in \Pi}\lim_{T \to \infty} \sup_{P \in \mathcal{P}_{\bm{\sigma}^2}} \sqrt{T}\,\mathrm{Regret}_{P}(\pi) \geq \frac{1}{\sqrt{e}}\left(\sigma(1) + \sigma(2)\right),
\end{align*}
where $e = 2.718\dots$ is Napier's constant.

Finally, we establish the worst-case upper bound for the simple regret of the Neyman allocation as follows:
\begin{align*}
    \limsup_{T \to \infty} \sup_{P \in \mathcal{P}_{\bm{\sigma}^2}} \sqrt{T}\, \mathrm{Regret}_P\left(\pi^{\mathrm{NA}}\right) \leq \frac{1}{\sqrt{e}}\left(\sigma(1) + \sigma(2)\right).
\end{align*}
This result proves that the Neyman allocation is asymptotically minimax optimal, as it achieves:
\begin{align*}
    \limsup_{T \to \infty} \sup_{P \in \mathcal{P}_{\bm{\sigma}^2}} \sqrt{T}\, \mathrm{Regret}_P\left(\pi^{\mathrm{NA}}\right) \leq \frac{1}{\sqrt{e}}\left(\sigma(1) + \sigma(2)\right) \leq \inf_{\pi \in \Pi}\lim_{T \to \infty} \sup_{P \in \mathcal{P}_{\bm{\sigma}^2}}\sqrt{T}\,\mathrm{Regret}_{P}(\pi).
\end{align*}

\subsection{Related work}
\label{sec:related}
Asymptotically optimal strategies have been extensively studied in the fixed-budget BAI problem. First, we note that the simple regret can be decomposed as:
\begin{align*}
    \mathrm{Regret}_{P_{\bmmu_0}}(\pi) = \mathrm{Regret}_{\bmmu_0}(\pi) = \left(\max_{b \in \{1, 2\}} \mu(b) - \min_{b \in \{1, 2\}} \mu(b)\right) \bbP_{\bmmu_0}\left(\widehat{a}_T^{\pi} \neq a^*(\bmmu_0)\right).
\end{align*}
Here, $\bbP_{\bmmu_0}\left(\widehat{a}_T^{\pi} \neq a^*(\bmmu_0)\right)$ is referred to as the \emph{probability of misidentification}, which is also a parameter of interest in BAI \citep{Kaufman2016complexity}. 
Since there are only two treatment arms, the absolute value of the gap $\max_{b \in \{1, 2\}} \mu(b) - \min_{b \in \{1, 2\}} \mu(b)$ is equivalent to the absolute value of the average treatment effect (ATE), i.e., $|\mu(1) - \mu(2)|$.

For simplicity, in this section, we assume without loss of generality that the best arm is arm $1$, i.e., $a^*(\bmmu_0) = 1$, so that $\max_{b \in \{1, 2\}} \mu(b) - \min_{b \in \{1, 2\}} \mu(b) = \mu(1) - \mu(2)$ and $\bbP_{\bmmu_0}\left(\widehat{a}_T^{\pi} \neq a^*(\bmmu_0)\right) = \bbP_{\bmmu_0}\left(\widehat{a}_T^{\pi} \neq 1\right)$.

In the evaluation of the simple regret, the balance between the ATE $\mu(1) - \mu(2)$ and the probability of misidentification $\bbP_{\bmmu_0}\left(\widehat{a}_T^{\pi} \neq 1\right)$ plays a key role. When evaluating the simple regret $\mathrm{Regret}_{\bmmu_0}(\pi)$ for each $P_{\bmmu_0}$, under well-designed algorithms, such as consistent strategies explained in Definition~\ref{def:consistent}, the probability of misidentification converges to zero as $T \to \infty$ with an order of $\exp(-TC(\bmmu_0))$, where $C(\bmmu_0) > 0$ is a parameter depending on $\bmmu_0$. 

Since the simple regret is the product of the ATE and the probability of misidentification, we have
\begin{align*}
    &\bbP_{\bmmu_0}\left(\widehat{a}_T^{\pi} \neq 1\right) \approx \exp\left(-T C(\bmmu_0)\right), \quad C(\bmmu_0) > 0, \\
    &\mathrm{Regret}_{\bmmu_0}(\pi) \approx \left(\mu(1) - \mu(2)\right) \exp\left(-T C(\bmmu_0)\right) = \text{gap} \times \text{misidentification probability},
\end{align*}
where $C(\bmmu_0)$ depends on $\bmmu_0$. In this asymptotic regime, if $\bmmu_0$ is independent of $T$, the probability of misidentification dominates the convergence of the simple regret since while the probability of misidentification converges to zero at an exponential rate, the gap is a fixed constant. It means that the influence of the ATE $\mu(1) - \mu(2)$ becomes negligible as $T \to \infty$. 

When the variances of the outcomes are known, the optimality of the Neyman allocation for the BAI problem has been shown using various approaches \citep{glynn2004large, kaufmann14}. Notably, \citet{Kaufman2016complexity} rigorously prove the optimality of the Neyman allocation for the probability of misidentification $\bbP_{\bmmu_0}\left(\widehat{a}_T^{\pi} \neq 1\right)$ under \emph{any} Gaussian distribution with finite variances in the case where $\bmmu_0$ is independent of $T$, as stated in the following proposition.

\begin{proposition}[Theorem~12 in \citet{Kaufman2016complexity} (informal)]
Assume that $w^*$ is known. Allocate treatment arm $1$ for the first $T_1 = T w^*(1)$ samples and treatment arm $2$ for the next $T_2 = T w^*(2)$ samples, where $T_1 + T_2 = T$. Using these samples, compute
\begin{align*}
    \widehat{\mu}_T^{\dagger}(1) \coloneqq \frac{1}{T_1} \sum_{t=1}^{T_1} Y_t, \quad \widehat{\mu}_T^{\dagger}(2) \coloneqq \frac{1}{T_2} \sum_{t=T_1 + 1}^T Y_t.
\end{align*}
Recommend $\widehat{a}_T^{\dagger} \coloneqq \arg\max_{a \in \{1, 2\}} \widehat{\mu}_T^{\dagger}(a)$ as the best arm. Then, for any Gaussian distribution $P \in \Bigcb{\mathcal{N}(\mu(1), \sigma^2(1)), \mathcal{N}(\mu(2), \sigma^2(2))\colon (\mu(1), \mu(2)) \in \bbR^2 }$ with finite variances $\sigma^2(1), \sigma^2(2) > 0$, independent of $T$, it holds for sufficiently large $T$ and any algorithm $\pi \in \Pi$ that
\[
    \bbP_{P}\Bigp{\widehat{a}_T^{\dagger} \neq a^*(P)} \leq \exp\p{-\frac{T \left(\mu_P(1) - \mu_P(2)\right)^2}{2 \left(\sigma(1) + \sigma(2)\right)^2}} \leq \bbP_{P}\left(\widehat{a}_T^{\pi} \neq a^*(P)\right),
\]
where $\mu_P(a) = \bbE_P[Y(a)]$ and $\Pi$ is the set of consistent strategies (Definition~\ref{def:consistent}).
\end{proposition}

The above result is stronger than minimax optimality because it holds for any distribution $P$, independent of $T$. 

However, the problem remains open when the variances are unknown and the distributions are non-Gaussian. Even under Gaussian distributions, variance estimation during the experiment introduces estimation error, which prevents achieving the same guarantees as \citet{Kaufman2016complexity}.

To address this challenge, the minimax framework plays a critical role. \citet{adusumilli2022minimax} tackle this issue and demonstrate that under local asymptotic normality and diffusion approximations, the Neyman allocation is minimax optimal for the simple regret. Similarly, \citet{kato2024locallyoptimalfixedbudgetbest,kato2024adaptivegeneralizedneymanallocation} show that in the small-gap regime, where $\mu_P(1) - \mu_P(2) \to 0$, the variance estimation error can be ignored for the probability of misidentification.

These studies, however, have notable limitations. \citet{adusumilli2022minimax} rely on local asymptotic normality and diffusion processes, which are approximations that restrict the underlying distributions. \citet{kato2024locallyoptimalfixedbudgetbest} avoid such approximations but focus on the small-gap regime, which may not align well with economic theory.

In this study, we establish minimax optimality for the simple regret without resorting to local asymptotic normality, diffusion processes, or the small-gap regime. We can estimate the variance, and our algorithms are asymptotically optimal for non-Gaussian distributions. Instead, we adopt the natural and widely-used minimax regret evaluation framework, which has strong connections to economic theory \citep{Manski2000, Manski2002, Manski2004, Stoye2009}. Notably, we show that strictly tight lower and upper bounds for the simple regret can be obtained without such approximations.

\section{The Neyman Allocation}
\label{sec:neyman}
This section introduces the Neyman allocation algorithm with the AIPW estimator. Our proposed algorithm uses the Neyman allocation in the allocation phase and the AIPW estimator in the recommendation phase. 

\subsection{The Neyman allocation in the allocation phase}
The Neyman allocation aims to allocate each treatment arm $a \in \{1, 2\}$ with probability $w^*(a)$, defined as:
\begin{align*}
    w^*(1) &\coloneqq \frac{\sigma(1)}{\sigma(1) + \sigma(2)}, \qquad w^*(2) \coloneqq 1 - w^*(1) = \frac{\sigma(2)}{\sigma(1) + \sigma(2)}.
\end{align*}

Since the variances $\sigma^2(1)$ and $\sigma^2(2)$ are unknown, they are estimated using observations collected during the experiment. 

In the first round ($t = 1$), a treatment arm is randomly allocated with equal probability $1/2$. For each round $t = 2, 3, \dots, T$, a treatment arm is allocated based on the estimated allocation probabilities $\widehat{w}_t$, defined as
\begin{align*}
    \widehat{w}(1) &\coloneqq \frac{\widehat{\sigma}_t(1)}{\widehat{\sigma}_t(1) + \widehat{\sigma}_t(2)}, \quad \widehat{w}(2) \coloneqq \frac{\widehat{\sigma}_t(2)}{\widehat{\sigma}_t(1) + \widehat{\sigma}_t(2)}.
\end{align*}

For each $a \in \{1, 2\}$, the variance estimator $\widehat{\sigma}^2_t(a)$ is constructed as follows:
\begin{align*}
    \widehat{\sigma}^2_t(a) &\coloneqq \begin{cases}
        \widetilde{\sigma}^2_t(a) & \text{if } \widetilde{\sigma}^2_t(a) > 0, \\
        \eta & \text{if } \widetilde{\sigma}^2_t(a) = 0,
    \end{cases}
\end{align*}
where $\widetilde{\sigma}^2_t(a)$ and the sample mean $\widetilde{\mu}_t(a)$ are given by
\begin{align*}
    \widetilde{\sigma}^2_t(a) &\coloneqq \frac{1}{\sum_{s=1}^{t-1} \mathbbm{1}[A_s = a]} \sum_{s=1}^{t-1} \mathbbm{1}[A_s = a]\left(Y_s - \widetilde{\mu}_t(a)\right)^2, \\
    \widetilde{\mu}_t(a) &\coloneqq \frac{1}{\sum_{s=1}^{t-1} \mathbbm{1}[A_s = a]} \sum_{s=1}^{t-1} \mathbbm{1}[A_s = a] Y_s.
\end{align*}
Here, $\eta \in (0, 1)$ is a small positive constant introduced to prevent division by zero. While the choice of $\eta$ does not affect the asymptotic properties, it may influence finite-sample performance, which is beyond the scope of this study.

\subsection{AIPW estimator in the recommendation phase}
After the allocation phase, using the observations $\{(A_t, Y_t)\}^T_{t=1}$, the conditional expected outcome $\mu_0(a)$ is estimated. For this estimation, the AIPW estimator is used, defined as follows for each $a \in \{1, 2\}$:
\begin{align*}
    \widehat{\mu}_T^{\mathrm{AIPW}}(a) \coloneqq \frac{1}{T} \sum_{t=1}^T\p{ \frac{\mathbbm{1}[A_t = a]\left(Y_t - \widetilde{\mu}_t(a)\right)}{\widehat{w}_t(a)} + \widetilde{\mu}_t(a)}.
\end{align*}

The AIPW estimator is known to be unbiased for $\mu_0(a)$ and achieves the smallest asymptotic variance among mean estimators. Its unbiasedness is based on the property that for $Z_t(a) \coloneqq \frac{\mathbbm{1}[A_t = a]\left(Y_t - \widetilde{\mu}_t(a)\right)}{\widehat{w}_t(a)} + \widetilde{\mu}_t(a) - \mu_0(a)$, $\{Z_t\}^T_{t=1}$ forms a martingale difference sequence; that is, 
\[
\bbE\sqb{Z_t(a)\mid \calF_{t-1}} =  \bbE\sqb{\frac{\mathbbm{1}[A_t = a]\left(Y_t - \widetilde{\mu}_t(a)\right)}{\widehat{w}_t(a)} + \widetilde{\mu}_t(a) - \mu_0(a)\mid \calF_{t-1}} = 0.
\]
This property significantly simplifies the theoretical analysis. Additionally, as shown later, since the variance of the mean estimator is the main factor influencing simple regret, reducing this variance directly enhances the overall performance of the algorithm. This type of estimator has been employed in existing studies, such as \citet{hadad2019} and \citet{Kato2020adaptive}.

By contrast, the sample mean $\widetilde{\mu}_t(a)$ is a biased estimator because $\bbE_{P_0}[\widetilde{\mu}_t(a)] = \mu_0(a)$ does not strictly hold. While the asymptotic properties of the AIPW estimator can also be applied to the sample mean, proving this requires more complex techniques. For instance, \citet{Hahn2011} demonstrate the asymptotic normality of the sample mean estimator by first showing its asymptotic equivalence to the AIPW estimator and then proving the asymptotic normality of the latter. However, their proof relies on stochastic equicontinuity, which is insufficient for our analysis since we also evaluate the large deviation property of the AIPW estimator. Although the sample mean performs better in finite samples, the AIPW estimator suffices when the focus is on the asymptotic properties of the algorithm.

Furthermore, the inverse probability weighting (IPW) estimator $\widehat{\mu}_T^{\mathrm{IPW}}(a) \coloneqq \frac{1}{T} \sum_{t=1}^T \frac{\mathbbm{1}[A_t = a] Y_t}{\widehat{w}_t(a)}$ 
is unbiased but has a larger variance compared to the AIPW estimator.

\section{Statistical models and lower bounds}
\label{sec:lower}
In this section, we derive a minimax lower bound. We first define a class of distributions considered in this study. Then, we present the minimax lower bound. 

\subsection{Location-shift models}
In this study, we consider the location-shift model with fixed unknown variances.
\begin{definition}[Location-shift models]
\label{def:location}
Fix $\bm{\sigma}^2 \coloneqq \{\sigma^2(1),\sigma^2(2),\dots, \sigma^2(K)\} \in \bbR^2$, which is a vector of variances unknown to us. Then, the location-shift model is defined as follows:
    \[\calP_{\bm{\sigma}^2} \coloneqq \cb{P_{\bmmu}\colon \bmmu \in \bbR^2,\ \Var_{\bmmu}(Y(a)) = \sigma(a)\ \forall a \in \{1, 2\},\ \eqref{cond1},\ \eqref{cond2},\ \mathrm{and}\ \eqref{cond3}},\]
where $\Var_{\bmmu}(\cdot)$ denotes a variance operator under $P_{\bmmu}$, and \eqref{cond1} are \eqref{cond2} are defined as follows:

\begin{enumerate}[topsep=0pt, itemsep=0pt, partopsep=0pt, leftmargin=*]
\renewcommand\labelenumi{(\theenumi)}
    \item \label{cond1} A distribution $P_{\mu_a, a}$ has a probability mass function or probability density function, denoted by $f_a(y\mid \mu_a)$. Additionally, $f_a(y\mid \mu_a) > 0$ holds for all $y \in \bbR$ and $\mu(a) \in \bbR$. 
    \item \label{cond2} For each $\mu_a \in \Theta$ and each $a\in[K]$, the Fisher information $I_a(\mu_a) > 0$ of $P_{\mu(a)}(a)$ exists.
    \item \label{cond3} Let $\ell_a\mu_a) = \ell_a\mu_a\mid y) = \log f(y\mid \mu(a))$ be the likelihood function of $P_{\mu(a)}(a)$, and $\dot{\ell}_a$, $\ddot{\ell}_a$, and $\dddot{\ell}_a$ be the first, second, and third derivatives of $\ell_a$. The likelihood functions $\big\{\ell_a\mu(a))\big\}_{a\in[K]}$ are three times differentiable and satisfy the following:
    \begin{enumerate}
        \item $\bbE_{P_{\mu(a)}(a)}\left[\dot{\ell}_a(\mu(a))\right] = 0$;
        \item $\bbE_{P_{\mu(a)}(a)}\left[\ddot{\ell}_a(\mu_a)\right] = -I_a(\mu_a) = 1/\sigma^2(a)$;
        \item For each $\mu(a) \in \Theta$, there exist a neighborhood $U(\theta)$ and a function $u(y\mid \mu(a)) \geq 0$, and the following holds:
        \begin{enumerate}
            \item $\left|\ddot{\ell}_a(\tau)\right| \leq u(y\mid \theta)\ \ \ \mathrm{for}\ U(\mu(a))$;
            \item $\bbE_{P_{\mu(a)}(a)}\left[u(Y\mid \mu(a))\right] < \infty$. 
        \end{enumerate}
    \end{enumerate}
\end{enumerate}
\end{definition}

 In this model, only mean parameters shift, while the variances are fixed. This model includes a normal distribution as a special case. 

\subsection{Minimax lower bound}
Next, we restrict the class of strategies to derive a tight lower bound. In this study, we consider consistent strategies defined as follows:

\begin{definition}[Consistent algorithm]
\label{def:consistent}
    We say that the class of strategies, $\Pi$, is the class of consistent strategies if the followings hold:
    \begin{itemize}
        \item For any $\pi \in \Pi$ and for any $P\in\calP_{\bm{\sigma}^2}$ such that $\mu_P(a)$ is independent of $T$, it holds that
    \[\lim_{T\to \infty}\bbP_{P}\Bigp{\widehat{a}^\pi_T = a^*(P)} = 1.\]
    \item For any $\pi \in \Pi$ and for any $P\in\calP_{\bm{\sigma}^2}$ whose mean satisfies $\mu_P(a) \geq c/\sqrt{T}$ for some positive constant $c$ independent of $T$, there exists $C > 0$ independent of $T$ such that
    \[\lim_{T\to \infty}\bbP_{P}\Bigp{\widehat{a}^\pi_T = a^*(P)} \geq C.\]
    \end{itemize}
\end{definition}

This definition implies that any strategies belonging to the set $\Pi$ returns the true best arm with probability one when the sample size $T$ is sufficiently large.

\begin{theorem}[Lower bounds]
\label{thm:minimax_lowerbound}
Fix $\bm{\sigma}^2 \in (0, \infty)^2$. Then, the following holds:
    \begin{align*}
        \inf_{\pi \in \Pi}\liminf_{T\to\infty} \sqrt{T} \sup_{P \in \calP_{\bm{\sigma}^2}} \mathrm{Regret}_P(\pi) \geq   \frac{1}{\sqrt{e}}\Bigp{\sigma\bigp{1} + \sigma\bigp{2}}.
\end{align*}
\end{theorem}
Here, $\sqrt{T}$ is a scaling factor. 

In other expression, we can write the statement as follows: any consistent algorithm $\pi\in\Pi$ satisfies that for any location shift model $P\in \calP_{\bm{\sigma}^2}$, the simple regret is lower bounded as
\[\mathrm{Regret}_P(\pi) \geq  \frac{\sigma\bigp{1} + \sigma\bigp{2}}{\sqrt{eT}} + o\p{\frac{1}{\sqrt{T}}}\quad (T\to\infty).\]

\subsection{Proof of the minimax lower bound}
In the derivation of the lower bound, we employ the change-of-measure arguments. These arguments involve comparing two distributions: the baseline hypothesis and the alternative hypothesis, to establish a tight lower bound. The change-of-measure approach is a standard method for deriving lower bounds in various problems, including nonparametric regression \citep{Stone1982}. Local asymptotic normality is one such technique frequently used in this context \citep{Vaart1998}. 

In the cumulative reward maximization of the bandit problem, the lower bound is derived using similar arguments and is widely recognized as a standard theoretical criterion in this area. This methodology provides a rigorous foundation for analyzing the theoretical performance limits of algorithms.

Let us denote the number of drawn arms by 
\[N_T(a) = \sum^T_{t=1}\mathbbm{1}[A_t = a].\] 
Then, we introduce the \emph{transportation} lemma, shown by \citet{Kaufman2016complexity}. 
\begin{proposition}[Transportation lemma. From Lemma~1 in \citet{Kaufman2016complexity}]
\label{prp:transport}
Let $P$ and $Q$ be two bandit
models with $K$ arms such that for all $a$, the marginal distributions $P(a)$ and $Q(a)$ of $Y(a)$ are mutually absolutely continuous. 
Then, we have
\[
\sum_{a\in\{1, 2\}} \bbE_P[N_T(a)] \mathrm{KL}(P(a),Q(a))\geq
\sup_{\mathcal{E} \in \mathcal{F}_T} \ d(\bbP_P(\mathcal{E}),\bbP_{Q}(\mathcal{E})),
\]
where $d(x,y):= x\log (x/y) + (1-x)\log((1-x)/(1-y))$ is the binary relative entropy, with the convention 
that $d(0,0)=d(1,1)=0$.
\end{proposition}
Here, $P$ corresponds to the baseline distribution, and $Q$ corresponds to the corresponding alternative distribution. 

It is well known that the KL divergence can be approximated by the Fisher information of a parameter when the parameter approaches zero. We summarize this property in the following proposition. 

\begin{proposition}[Proposition 15.3.2. in \citet{Duchi2023} and Theorem 4.4.4 in \citet{calin2014geometric}]
\label{prp:kl_fisher}
    We have
    \begin{align*}
        \lim_{\nu(a) \to \mu(a)}\frac{1}{\left(\mu(a) - \nu(a)\right)^2}\mathrm{KL}(\mu(a), \nu(a)) = \frac{1}{2}I(\mu_a)
    \end{align*}
\end{proposition}

Then, using Proposition~\ref{prp:transport} and \ref{prp:kl_fisher}, we prove the lower bound in Theorem~\ref{thm:minimax_lowerbound} as follows.
\begin{proof}[Proof of Theorem~\ref{thm:minimax_lowerbound}]
We decompose the simple regret as follows:
\[\max_{P \in \calP_{\bm{\sigma}^2}} \mathrm{Regret}_P(\pi) = \max_{\widetilde{a}\in\{1, 2\}}\max_{P \in \calP_{\bm{\sigma}^2, \widetilde{a}}} \mathrm{Regret}_P(\pi),\]
where $\calP_{\bm{\sigma}^2, \widetilde{a}}$ is subset of $\calP_{\bm{\sigma}^2}$ whose best arm is \emph{not} $\widetilde{a}$:
\[\calP_{\bm{\sigma}^2, \widetilde{a}} \coloneqq \Bigcb{P \in P_{\bm{\sigma}^2}\colon \argmax_{a \in \{1, 2\}} \bbE\bigsqb{Y(a)} \neq \widetilde{a} }.\]
Here, $\widetilde{a}$ corresponds to the best arm of a baseline hypothesis. 

First, we investigate the case with $\widetilde{a} = 1$. 
Given $P_{\bmnu} \in \calP_{\bm{\sigma}^2, 1}$, we can lower bound $\mathrm{Regret}_P(\pi)$ as follows:
\begin{align*}
    \mathrm{Regret}_{P_{\bmnu}}(\pi) = \Regret_{\bmnu}(\pi) &= \Bigp{\nu(2) - \nu(1)}\bbP_{\bmnu}\Bigp{\widehat{a}^\pi_T = 1}.
\end{align*}

We consider the lower bound of $\bbP_{\bmnu}\Bigp{\widehat{a}^\pi_T = 1}$. We define the baseline model $P_{\bmmu}$ with a parameter $\bmmu \in \bbR^2$, defined as follows:
\begin{align*}
    \mu(b) = \begin{cases}
     \eta &\mathrm{if}\ \ b = 1\\
      0 &\mathrm{if}\ \ b = 2
    \end{cases},
\end{align*}
where $\eta > 0$ is a small positive value. We take $\eta \to 0$ at the last step of the proof, indepdently of $T$. Corresponding to the baseline model, the best arm under the alternative hypophysis is given as $2$. We set a parameter $\bmnu \in \bbR^2$ of the alternative model $P_{\bmnu}$ as 
\begin{align*}
    \nu(b) = \begin{cases}
     - \sqrt{\frac{1}{T}}\sigma(1) &\mathrm{if}\ \ b = 1\\
      \sqrt{\frac{1}{T}}\sigma(2) &\mathrm{if}\ \ b = 2
    \end{cases}.
\end{align*}

Let $\mathcal{E}$ be the event $\widehat{a}^\pi_T = 2$. 
Between the baseline distribution $P_{\bmmu}$ and the alternative hypothesis $P_{\bm{\nu}}$, from Proposition~\ref{prp:transport}, we have
\[
\sum_{a\in\{1, 2\}} \bbE_{\bmmu}[N_T(a)] \mathrm{KL}(P_{\mu(a)}(a),P_{\nu(a)}(a))\geq
\sup_{\mathcal{E} \in \mathcal{F}_T} \ d(\bbP_{\bmmu}(\mathcal{E}),\bbP_{\bmnu}(\mathcal{E})).
\]

Under any consistent algorithm $\pi \in \Pi^{\mathrm{const}}$, we have $\bbP_{\bmmu}(\mathcal{E}) \to 0$ and $\bbP_{\bmnu}(\mathcal{E}) \geq C$ as $T \to \infty$, where $C > 0$ is a constant independent of $T$. 

Therefore, for any $\varepsilon > 0$, there exists $T(\epsilon)$ such that for all $T \geq T(\varepsilon)$, it holds that
\[0\leq \bbP_{\bmmu}(\mathcal{E}) \leq \varepsilon \leq \bbP_{\bmnu}(\mathcal{E}) \leq 1.\]

Since $d(x,y)$ is defined as $d(x,y):= x\log (x/y) + (1-x)\log((1-x)/(1-y))$, we have
\begin{align*}
    &\sum_{a\in\{1, 2\}} \bbE_{\bmmu}[N_T(a)] \mathrm{KL}(P_{a, \mu_a},P_{a, \nu_a}) \geq d(\varepsilon,\bbP_{\bmnu}(\mathcal{E}))\\
    &\ \ \ = \varepsilon\log \left(\frac{\varepsilon}{\bbP_{\bmnu}(\mathcal{E})}\right) + \left(1 - \varepsilon\right)\log \left(\frac{1 - \varepsilon}{1 - \bbP_{\bmnu}(\mathcal{E})}\right)\\
    &\ \ \ \geq \varepsilon\log \left(\varepsilon\right) + \left(1 - \varepsilon\right)\log \left(\frac{1 - \varepsilon}{1 - \bbP_{\bmnu}(\mathcal{E})}\right)\\
    &\ \ \ \geq \varepsilon\log \left(\varepsilon\right) + \left(1 - \varepsilon\right)\log \left(\frac{1 - \varepsilon}{\bbP_{\bmnu}(\widehat{a}^\pi_T = a^*(P_{\bmmu}))}\right).
\end{align*}
Note that $\varepsilon$ is closer to $\bbP_{\bmnu}(\mathcal{E})$ than $\bbP_{\bmmu}(\mathcal{E})$; therefore, we used 
$d(\bbP_{\bmmu}(\mathcal{E}),\bbP_{\bmnu}(\mathcal{E})) \geq d(\varepsilon,\bbP_{\bmnu}(\mathcal{E}))$. 

Therefore, we have 
\begin{align*}
    \bbP_{\bmnu}(\widehat{a}^\pi_T = a^*(P_{\bmmu})) \geq \exp\left(- \frac{1}{1 - \varepsilon}\sum_{a\in\{1, 2\}} \bbE_P[N_T(a)] \mathrm{KL}(P_{a, \mu_a},P_{a, \nu_a}) + \frac{\varepsilon}{1 - \varepsilon}\log \left(\varepsilon\right)\right) + 1 - \varepsilon 
\end{align*}

Here, from Proposition~\ref{prp:kl_fisher}, for any $\varepsilon > 0$, there exists $\Xi_a(\varepsilon)$ such that for all $- \Xi_a(\varepsilon) < \xi_a \coloneqq - \mu_a + \nu_a < \Xi_a(\varepsilon)$, the following holds:
\begin{align*}
    \mathrm{KL}(\mu_a, \mu_a + \xi_a) \leq \frac{\xi^2_a}{2}I\big(\mu_a\big) + \varepsilon \xi^2_a = \frac{\xi^2_a}{2\sigma_a(\mu_a)} + \varepsilon \xi^2_a,
\end{align*}
where we used $I\big(\mu_a\big) = \sigma^2(a)$. 

Then, we have
\begin{align*}
    &\bbP_{\bmnu}(\widehat{a}^\pi_T = a^*(\bmmu))     \geq \bigp{1 - \varepsilon}\exp\left(- \frac{1}{1 - \varepsilon}\sum_{a\in\{1, 2\}}\bbE_{\bmmu}\sqb{N_T(a)} \mathrm{KL}(P_{a, \mu(a)},P_{a, \nu_a}) + \frac{\varepsilon}{1 - \varepsilon}\log \left(\varepsilon\right)\right)\\
    &\geq \bigp{1 - \varepsilon}\exp\left(- \frac{1}{1 - \varepsilon}\sum_{a\in\{1, 2\}} \bbE_{\bmmu}\sqb{N_T(a)}\p{\frac{\left(\mu(a) - \nu(a)\right)^2}{2\sigma^2(a)} + \varepsilon \left(\mu(a) - \nu(a)\right)^2} + \frac{\varepsilon}{1 - \varepsilon}\log \left(\varepsilon\right)\right).
\end{align*}

Let $\bbE_{\bmmu}\sqb{N_T(a)}$ be denoted by $Tw_{\bmmu}(a)$. Then, the following inequality holds:
\begin{align*}
    &\bbP_{\bmnu}(\widehat{a}^\pi_T = a^*(\bmnu))\\
    &\geq \bigp{1 - \varepsilon}\exp\left(- \frac{1}{1 - \varepsilon}\sum_{a\in\{1, 2\}} \left(Tw_{\bmmu}(a)\p{\frac{\left(\mu(a) - \nu(a)\right)^2}{2\sigma^2(a)} + \varepsilon \p{\mu(a) - \nu(a)}^2}\right) + \frac{\varepsilon}{1 - \varepsilon}\log \left(\varepsilon\right)\right).
\end{align*}

We set $w_{\bmmu}(a)$ as
\begin{align*}
    w_{\bmmu}(b) = \begin{cases}
     \frac{\sigma(1)}{\sigma(1) + \sigma(2)} &\mathrm{if}\ \ b = 1\\
      \frac{\sigma(2)}{\sigma(1) + \sigma(2)} &\mathrm{if}\ \ b = 2
    \end{cases}.
\end{align*}
By substituting them, we have
\begin{align*}
    &\max_{P \in \calP_{\bm{\sigma}^2}} \mathrm{Regret}_P(\pi)\\
    &\geq \Bigp{\nu(2) - \nu(1)}\bigp{1 - \varepsilon}\\
    &\ \ \ \ \ \ \ \ \ \ \exp\left(- \frac{1}{1 - \varepsilon}\sum_{a\in\{1, 2\}} Tw_{\bmmu}(a)\p{\frac{\left(\mu(a) - \nu(a)\right)^2}{2\sigma^2(a)} + \varepsilon \p{\mu(a) - \nu(a)}^2} + \frac{\varepsilon}{1 - \varepsilon}\log \left(\varepsilon\right)\right)\\
    &= \sqrt{\frac{1}{T}}\Bigp{\sigma(1) + \sigma(2)}\bigp{1 - \varepsilon}\\
    &\ \ \ \ \ \ \ \ \ \ \exp\left(- \frac{1}{1 - \varepsilon} \p{\bigp{1 + g(\eta)} / 2 + \varepsilon T\p{\sqrt{\frac{1}{T}}\sigma(1) + \eta}^2 + \varepsilon T\p{\sqrt{\frac{1}{T}}\sigma(2)}^2} + \frac{\varepsilon}{1 - \varepsilon}\log \left(\varepsilon\right)\right)\\
    &= \sqrt{\frac{1}{T}}\Bigp{\sigma(1) + \sigma(2)}\bigp{1 - \varepsilon}\Biggp{\exp\left(- \frac{1}{2\bigp{1 - \varepsilon}} \Bigp{1 + \widetilde{g}(\eta, \varepsilon)} + \frac{\varepsilon}{1 - \varepsilon}\log \left(\varepsilon\right)\right) + 1 - \varepsilon},
\end{align*}
where $g(\eta)$ and $\widetilde{g}(\eta, \varepsilon)$ are terns converging to zero as $\eta \to 0$ and $\varepsilon \to 0$.

Then, for any consistent algorithm $\pi$, by letting $T\to \infty$, $\varepsilon \to 0$, and $\eta \to 0$, we have
\[\limsup_{T\to\infty} \sqrt{T} \max_{P \in \calP_{\bm{\sigma}^2}} \mathrm{Regret}_P(\pi) \geq   \Bigp{\sigma\bigp{1} + \sigma\bigp{2}}\exp\left(- 1/2\right).\]

\end{proof}

\section{Upper bound and minimax optimality}
\label{sec:upper}
In this subsection, we establish an upper bound on the simple regret for the Neyman allocation algorithm. The bound demonstrates that the Neyman allocation achieves asymptotic minimax optimality. Specifically, the simple regret under this algorithm matches the minimax lower bound including the constant terms, not only for the rate regarding the sample size.

First, we derive the following worst-case upper bound for the simple regret of the Neyman allocation.
\begin{theorem}
\label{thm:upper_simpleregret}
For the Neyman allocation, the simple regret is upper bounded as 
\begin{align*}
    \limsup_{T \to \infty} \sup_{P \in \calP_{\bm{\sigma}^2}} \sqrt{T} \mathrm{Regret}_P\p{\pi^{\mathrm{NA}}} \leq \frac{1}{\sqrt{e}}\Bigp{\sigma\bigp{1} + \sigma\bigp{2}}.
\end{align*}
\end{theorem}

We upper bound the simple regret of the Neyman allocation algorithm in Theorem~\ref{thm:upper_simpleregret}. The results in the lower bound (Theorem~\ref{thm:minimax_lowerbound}) and the upper bound (Theorem~\ref{thm:upper_simpleregret}) imply the asymptotic minimax optimality.

\begin{corollary}[Asymptotic minimax optimality]
    Under the same conditions in Theorems~\ref{thm:minimax_lowerbound} and \ref{thm:upper_simpleregret}, it holds that
    \begin{align*}
        \limsup_{T \to \infty} \sup_{P \in \calP_{\bm{\sigma}^2}} \sqrt{T} \mathrm{Regret}_P\p{\pi^{\mathrm{NA}}} \leq \frac{1}{\sqrt{e}}\Bigp{\sigma\bigp{1} + \sigma\bigp{2}} \leq \min_{\pi \in \Pi}\liminf_{T\to\infty} \sqrt{T} \sup_{P \in \calP_{\bm{\sigma}^2}} \mathrm{Regret}_P(\pi).
    \end{align*}
\end{corollary}

This result shows that the \emph{exact} asymptotic minimax optimality of the Neyman allocation.

\paragraph{Proof of Theorem~\ref{thm:upper_simpleregret}.}
We present the proof of Theorem~\ref{thm:upper_simpleregret}. The proof is primarily based on the following lemma from \citet{kato2024locallyoptimalfixedbudgetbest}. 

\begin{lemma}
\label{lem:kato}
    Under $P_0$, for all $a \in \{1, 2\}\backslash \{a^*_0\}$ and for all $\epsilon > 0$, there exists $t(\epsilon) > 0$ such that for all $T > t(\epsilon)$, there exists $\underline{\delta}_T(\epsilon) > 0$ such that for all $0 < \mu_0\bigp{a^*_0} - \mu_0\bigp{a} < \underline{\delta}_T(\epsilon)$, the following holds:
    \begin{align*}
       &\bbP_{P_0}\Bigp{\widehat{\mu}^{\mathrm{AIPW}}_{T}\bigp{a^*_0} \leq \widehat{\mu}^{\mathrm{AIPW}}_{T}\bigp{a}} \leq \exp\p{ - \frac{T\Bigp{\mu_0\bigp{a^*_0} - \mu_0\bigp{a}}^2}{2\Bigp{\sigma(1) + \sigma(2)}^2} + \epsilon\Bigp{\mu_0\bigp{a^*_0} - \mu_0\bigp{a}}^2 T}.
    \end{align*}

\end{lemma}

\begin{proof}[Proof of Theorem~\ref{thm:upper_simpleregret}.]
We decompose the simple regret as 
\[\max_{P \in \calP_{\bm{\sigma}^2}} \mathrm{Regret}_P\p{\pi^{\mathrm{NA}}} = \max_{a^\dagger \in \{1, 2\}}\max_{P \in \calP_{\bm{\sigma}^2, a^\dagger}} \mathrm{Regret}_P\p{\pi^{\mathrm{NA}}}.\]

We consider the case where the data is generated from $P \in \calP_{\bm{\sigma}^2, a^\dagger}$. From Lemma~\ref{lem:kato}, for each $P \in \calP_{\bm{\sigma}^2, a^\dagger}$, for all $a \in \{1, 2\}\backslash \{a^*(P)\}$, and for all $\epsilon > 0$, there exists $t(\epsilon) > 0$ such that for all $T > t(\epsilon)$, there exists $\underline{\delta}_T(\epsilon) > 0$ such that for all $0 < \mu_0\bigp{a^*_0} - \mu_0\bigp{a} < \underline{\delta}_T(\epsilon)$, the following holds:
    \begin{align*}
       &\bbP_{P}\Bigp{\widehat{\mu}^{\mathrm{AIPW}}_{T}\bigp{a^*(P)} \leq \widehat{\mu}^{\mathrm{AIPW}}_{T}\bigp{a}} \leq \exp\p{ - \frac{T\Bigp{\mu_0\bigp{a^*(P)} - \mu_0\bigp{a}}^2}{2\Bigp{\sigma(1) + \sigma(2)}^2} + \epsilon\Bigp{\mu_0\bigp{a^*(P)} - \mu_0\bigp{a}}^2 T}.
    \end{align*}

Therefore, we have
\begin{align*}
    &\mathrm{Regret}_{P_0}\p{\pi^{\mathrm{NA}}}\leq \Bigp{\mu_0\bigp{a^*(P_0)} - \mu_0\bigp{a^\dagger}}\exp\p{ - \frac{T\Bigp{\mu_0\bigp{a^*_0} - \mu_0\bigp{a^\dagger}}^2}{2\Bigp{\sigma(1) + \sigma(2)}^2} + \epsilon\Bigp{\mu_0\bigp{a^*_0} - \mu_0\bigp{a^\dagger}}^2 T}.
\end{align*}
where $a^\dagger \neq a^*(P_0)$. Taking the maximum over $P_0$ is equal to solve the following problem: 
\begin{align*}
    &\max_{(\mu_0\bigp{a^*_0} - \mu_0\bigp{a^\dagger}) \in \bbR}\Bigp{\mu_0\bigp{a^*(P_0)} - \mu_0\bigp{a^\dagger}}\exp\p{ - \frac{T\Bigp{\mu_0\bigp{a^*_0} - \mu_0\bigp{a^\dagger}}^2}{2\Bigp{\sigma(1) + \sigma(2)}^2}},
\end{align*}
where we ignored $\epsilon\Bigp{\mu_0\bigp{a^*_0} - \mu_0\bigp{a^\dagger}}^2 T$ since it is ignorable at the limit of $T \to \infty$. Then, the maximizer is given as 
\begin{align*}
    \mu^*_0\bigp{a^*_0} - \mu^*_0\bigp{a^\dagger} = \frac{\sigma(1) + \sigma(2)}{\sqrt{T}}.
\end{align*}
By substituting this maximizer into the regert upper bound, we complete the proof. 
\end{proof}

\section{Extension to Bernoulli Distributions}
In this section, we extend our results to the case where the outcomes follow Bernoulli distributions. We find that the Neyman allocation does not outperform the uniform allocation, which assigns an equal number of samples to each treatment arm.

When considering Bernoulli distributions, the variances depend on the means. Specifically, if $\mu(1) - \mu(2) \to 0$ and $\mu \in [0, 1]$ such that $\mu \approx \mu(1) \approx \mu(2)$, the variances of the outcomes for both treatment arms are given by $\mu(1 - \mu)$, which achieves its maximum value of $0.5$.

Using this property of the Bernoulli distribution, we can establish the following lower bound as a corollary of Theorems~\ref{thm:minimax_lowerbound} and \ref{thm:upper_simpleregret}.

\begin{corollary}[Minimax Lower Bound under Bernoulli Distributions]
The following holds:
\begin{align*}
    \inf_{\pi \in \Pi} \liminf_{T \to \infty} \sqrt{T} \sup_{P \in \mathcal{P}^{\mathrm{Bernoulli}}} \mathrm{Regret}_P(\pi) \geq 2\sqrt{\frac{5}{e}}.
\end{align*}
\end{corollary}

We now consider the following uniform allocation algorithm (assuming $T$ is even for simplicity): for the first $T/2$ samples, we allocate treatment arm $1$, and for the next $T/2$ samples, we allocate treatment arm $2$. The uniform allocation algorithm achieves the following upper bound on the simple regret using the Chernoff bound.

\begin{theorem}[Simple Regret of Uniform Allocation]
For the uniform allocation, the simple regret is upper bounded as:
\begin{align*}
    \limsup_{T \to \infty} \sup_{P \in \mathcal{P}_{\bm{\sigma}^2}} \sqrt{T} \mathrm{Regret}_P\left(\pi^{\mathrm{NA}}\right) \leq 2\sqrt{\frac{5}{e}}.
\end{align*}
\end{theorem}

Thus, the uniform allocation is asymptotically minimax optimal for the simple regret.

Notably, the Neyman allocation achieves the same simple regret as the uniform allocation. This result can be intuitively understood as follows: in the limit where $\mu(1) - \mu(2) \to 0$, the variances of the two treatment arms become equal. Consequently, the Neyman allocation reduces to allocating an equal number of samples to each arm, which is equivalent to the uniform allocation.

We conclude that the Neyman allocation is as efficient as the uniform allocation in the case of Bernoulli distributions. This result implies that no algorithm can outperform the uniform allocation under Bernoulli distributions, making the Neyman allocation unnecessary in this setting. This conclusion is consistent with previous findings by \citet{kaufmann14,Kaufman2016complexity}, \citet{wang2023uniformly}, and \citet{kato2024adaptivegeneralizedneymanallocation}. Furthermore, \citet{horn2022comparisonmethodsadaptiveexperimentation} empirically report that the exploration sampling algorithm proposed by \citet{Kasy2021} performs similarly to the uniform allocation, a result that is theoretically supported by both our findings and the existing literature.

\section{Conclusion}
In this study, we addressed the fixed-budget BAI problem under the challenging setting of unknown variances. By introducing the Neyman allocation algorithm combined with the AIPW estimator, we proposed an asymptotically minimax optimal solution.

Our contributions are twofold. First, we derived the minimax lower bound for the simple regret, establishing a theoretical benchmark for any consistent algorithm. Second, we proved that the simple regret of the Neyman allocation algorithm matches this lower bound, including the constant term, not just the rate. This result demonstrates that the Neyman allocation achieves asymptotic minimax optimality even without assumptions such as local asymptotic normality, diffusion processes, or small-gap regimes.

The AIPW estimator played a crucial role in achieving this result, as it reduces the variance of the mean estimation, which directly impacts the simple regret. By carefully handling the variance estimation during the adaptive experiment, we showed that the estimation error does not compromise the asymptotic guarantees.

Our findings contribute to both the theoretical understanding and practical application of adaptive experimental design. Future research could explore the extension of these results to multi-armed settings or investigate the finite-sample behavior of the proposed algorithm to complement the asymptotic analysis.

\bibliography{arXiv.bbl}

\bibliographystyle{tmlr}

\onecolumn

\appendix



\end{document}